\documentclass[double]{IEEEtran}

\usepackage{latexsym,,amssymb,amsmath,graphicx,epsf,cite,bbm,float}
\usepackage{ifpdf}
\usepackage{epstopdf}

\usepackage{algorithm,algorithmic}
\usepackage{amsmath,amssymb,bm}
\usepackage{amsfonts,dsfont,color,bbm,subcaption}

\newcommand{\be}{\begin{equation}}
\newcommand{\ee}{\end{equation}}
\newcommand{\bea}{\begin{eqnarray}}
\newcommand{\eea}{\end{eqnarray}}

\newcommand{\MB}{\left[\begin{array}}
\newcommand{\ME}{\end{array}\right]}

\newcommand{\ei}{\end{itemize}}
\newcommand{\bi}{\begin{itemize}}

\newcommand{\Mt}[1][t]{\mathbbm{M}_{#1}}

\usepackage{amsmath,amsthm}

\newcommand{\E}{\mathbb{E}}
\newcommand\Tau{\mathcal{T}}

\newtheorem{theorem}{Theorem}
\newtheorem{lemma}[]{Lemma}
\newtheorem{corollary}[]{Corollary}
\newtheorem{remark}[]{Remark}
\begin{document}

\title{A Generalized Online Algorithm for Translation and Scale Invariant Prediction with Expert Advice} 
\author{Kaan Gokcesu, Hakan Gokcesu}
\maketitle

\begin{abstract}
	In this work, we aim to create a completely online algorithmic framework for prediction with expert advice that is translation-free and scale-free of the expert losses. Our goal is to create a generalized algorithm that is suitable for use in a wide variety of applications. For this purpose, we study the expected regret of our algorithm against a generic competition class in the sequential prediction by expert advice problem, where the expected regret measures the difference between the losses of our prediction algorithm and the losses of the 'best' expert selection strategy in the competition. We design our algorithm using the universal prediction perspective to compete against a specified class of expert selection strategies, which is not necessarily a fixed expert selection. The class of expert selection strategies that we want to compete against is purely determined by the specific application at hand and is left generic, which makes our generalized algorithm suitable for use in many different problems. We show that no preliminary knowledge about the loss sequence is required by our algorithm and its performance bounds, which are second order expressed in terms of sums of squared losses. Our regret bounds are stable under arbitrary scalings and translations of the losses. 
\end{abstract}

\section{Introduction}\label{sec:intro}
	In machine learning literature, the study of prediction with expert advice and online forecasting in adversarial scenarios has received considerable attention, where the goal is to minimize (or maximize) a certain loss (or reward) in a given environment \cite{cesabook}. This area of online learning is heavily investigated in various fields from game theory \cite{chang,tnnls1}, control theory \cite{sw2,sw4,tnnls3}, decision theory \cite{tnnls4} to computational learning theory \cite{comp1,comp2}. Because of its universal prediction perspective \cite{merhav}, it has been considerably utilized in data and signal processing \cite{gIncremental,sw3,signal2,moon,signal1,sw5}, especially in sequential prediction and estimation problems \cite{gHierarchical,ozkan,singer,singer2} such as the problem of density estimation and anomaly detection \cite{gAnomaly,willems,coding1,gDensity,coding2}. Moreover, it has numerous applications in multi-agent systems \cite{sw1,vanli,tekin1}, specifically, the reinforcement learning \cite{auerExp,bandit1,bandit2,audibert,tekin2,gBandit,exptrade,auer,auerSelf}.
		
	In the problem of prediction with expert advice, we have a set of $M$ actions (expert advice, e.g., algorithms) that we can take on a certain task. At each round of the decision process, we select one of these actions of the experts and receive its loss (or gain). One of the goals of the research in this area is the design of randomized online algorithms that achieve a low 'regret', i.e., algorithms that are able to minimize the difference between their expected loss and the loss of a strategy of expert selection \cite{littlestone1994, vovk1998}.
	
	We start by considering the case where a forecaster repeatedly assigns probabilities to a fixed set of actions, and after each assignment, the actual loss associated to each action is revealed and new losses are set for the next round \cite{freund1997}. We study the expert selection problem in an online setting, where we operate continuously on a stream of observations from a possibly nonstationary, chaotic or even adversarial environment. Hence, we assume no statistical assumptions on this loss sequence (this is done so that the results are universal, i.e., guaranteed to hold in an individual sequence manner). The forecaster’s loss on each round is the average loss of actions for that round, where the average is computed according to the forecaster’s current probability assignment. Since we have no statistical assumptions on the losses of the experts, we define our performance with respect to a competing class of strategies and investigate the expert selection problem from a competitive algorithm perspective \cite{merhav}. The goal of the forecaster is to achieve, on any sequence of loss, a cumulative loss that is close to the lowest cumulative loss among all expert selection strategies in our competition class (e.g., if the competition is against fixed expert selections, we compare against the expert with the best cumulative loss) \cite{cesa2007}. The difference between the cumulative loss of our forecaster and the best strategy (on the same loss sequence) is 'regret' \cite{cesabook}. 
	
	For the case of fixed competition, the most basic approach, obtained via the exponentially weighted average forecaster of \cite{littlestone1994} and \cite{vovk1998}, gives a zeroth order regret (where the regret bounds are dependent on the universal loss range and the number of rounds). In the special case of “one-sided games”, when all losses have the same sign, \cite{freund1997} showed that the algorithm in \cite{littlestone1994} can be used to obtain a first order regret bound (where the regret bounds are dependent on the sum of the losses). In \cite{allenberg2004}, a direct analysis on the signed games shows that weighted majority achieves the first order regret without any need for a one-sided loss game. Even though the approaches up to now are scale-free, they are neither translation-free nor parameter-free (since a priori knowledge about the losses are needed). These shortcomings are solved by \cite{cesa2007}, where they showed second order regret bounds (where the regret bounds are dependent on the sum of squared losses) for signed games and improve upon the previous bounds while also eliminating the need for a priori information. Thus, their algorithm is translation-free, scale free and also parameter-free. Nonetheless, their competition class is limited (mainly focused on the fixed expert selection strategies). There are variants in literature to deal with different applications, but, because of its nature, competing against arbitrary expert selection strategies is nontrivial unless you treat each strategy as an expert itself. However, since each such strategy constitutes a predetermined expert selection sequence, naively treating each strategy as an expert may lead to the mixture of up to $M^T$ strategies in a game of length $T$ and $M$ experts. Hence, in general, this naive approach would be difficult to implement for a scenario with a large competition class.
	
	To this end, we improve upon the previous works to provide an algorithmic framework to compete against arbitrary expert selection strategies with second order regret bounds. Our algorithmic framework can straightforwardly implement the desired competition class (in accordance with the needs of the problem at hand) in a scalable and tractable manner. We define our performance (i.e., regret) with respect to the best strategy (minimum loss) in that class. Since, in the competitive algorithm perspective we do not need to explicitly know the actions (experts) we are presented with (each expert can even be separately running algorithms that learn throughout time), the only prior knowledge we need about the experts is that there are $M$ options (whatever they may be) that we can select from, and what kind of expert selection strategies we want to compete against. Our algorithm works such that, at each time $t$, the action is chosen solely based on the sequential performance of the options (experts or strategies themselves). 
	
	The organization of the paper is as follows. In Section \ref{sec:problem}, we first describe the expert selection problem. Then, in Section \ref{sec:method}, we detail the methodology and our algorithmic framework. We provide the performance results and regret analysis in Section \ref{sec:regret}. Finally, in Section \ref{sec:example}, we demonstrate the construction of the algorithm using our framework with an example application and finish with some concluding remarks in Section \ref{sec:conc}. The detailed proofs of the results in Section \ref{sec:regret} are provided in appendix at the end.

\section{Problem Description}\label{sec:problem}
In this paper, we study the expert selection problem where we have $M$ experts such that $m\in\{1,\ldots,M\}$ and randomly select one of them at each round $t$. We select our expert according to our selection probabilities
\begin{align}
	p_t\triangleq[p_{t,1},\ldots,p_{t,M}],\label{eq:pt}
\end{align}
where our selection is $i_t\in\{1,\ldots,M\}$ such that
\begin{align}
	i_t\sim p_t.
\end{align}
Based on our online selection 
\begin{align}
	\{i_t\}_{t\geq1},\enspace i_t\in\{1,2,\ldots,M\},
\end{align}
we incur the loss of the selected experts
\begin{align}
	 \{{l_{t,i_t}}\}_{t\geq1},
\end{align}
where we do not assume anything about the losses before selecting our expert at time $t$.

In a $T$ round game, we define $I_T$ as the row vector containing the user selections up to time $T$ as 
\begin{align}
	I_T=[i_1,\ldots,i_T],
\end{align}
and the loss sequence of $I_T$ as
\begin{align}
	L_{I_T}=[l_{t,{i_1}},\ldots,l_{t,i_T}].
\end{align}
Similarly, we define the variable $S_T$ as the row vector representing a deterministic expert selection sequence of length $T$ as
\begin{align}
	S_T=[s_1,\ldots,s_T].\label{eq:St}
\end{align}
such that each $s_t\in\{1,2,\ldots,M\}$ for all $t$. In the rest of the paper, we refer to each such deterministic expert selection sequence, $S_T$, as a strategy. Hence, the loss sequence of the strategy $S_T$ is
\begin{align}
	L_{S_T}=[l_{t,{s_1}},\ldots,l_{t,s_T}].
\end{align}
We denote the cumulative loss at time $T$ of $I_T$ by 
\begin{align}
	{C_{I_T}=sum(L_{I_T})=\sum_{t=1}^T l_{t,i_t}},
\end{align}
and similarly the cumulative loss at time $T$ of any $S_T$ by 
\begin{align}
{C_{S_T}=sum(L_{S_T})=\sum_{t=1}^T l_{t,s_t}}.
\end{align}
Since we assume no statistical assumptions on the loss sequence, we define our performance with respect to any strategy $S_T$ that we want to compete against.
We use the notion of regret to define our performance against any strategy $S_T$ as
\begin{align}
R_{S_T}&\triangleq C_{I_T}-C_{S_T}=\sum_{t=1}^T l_{t,i_t}-\sum_{t=1}^T l_{t,s_t},\label{RT1}
\end{align}
where we denote the regret accumulated in $T$ rounds against $S_T$ as $R_{S_T}$. Our goal is to create an algorithm with expected regret bounds that depends on how hard it is to learn the strategy $S_T$. %The quantification of the hardness to learn the strategy $S_T$ is left up to design in our general framework. 

\section{Methodology}\label{sec:method}
	To construct our framework and better convey our methodology, we first consider the trivial approach of treating each strategy as an expert themselves (however intractable it may be). Hence, to produce the probabilities $p_t$ given in \eqref{eq:pt}, we universally combine each of the strategies $S_t\in\Mt$ at time $t$, where {$\Mt$} is the class of all strategies up to time $t$, and its size is $M^t$, i.e., {$|\Mt|=M^t$} (and $S_t$ is similarly defined as in \eqref{eq:St}).
	
	Our algorithm fundamentally works by assigning each of these strategies, $S_t$, a weight $w_{S_t}$ that shows our 'trust' on that particular strategy. Using these weights, we create our probability simplex $p_t$. Hence, to make our selection at time $t$, for each expert $m$, we need to find the strategies among all the $M^{t}$ strategies that suggest $m$ at round $t$ and sum their assigned weights to create the weight at time $t$ for the expert $m$, i.e.,
	\begin{align}
	w_{t,m}\triangleq\sum_{S_t(t:t)=m}^{} w_{S_t},\label{wmt}
	\end{align}
	where $S_t(i\!:\!j)$ is the vector consisting of $i^{th}$ through $j^{th}$ elements of $S_t$, e.g., $S_t(t:t)=s_t$, which is the expert selection of the strategy $S_t$ at time $t$. By summing the probabilities of strategies that suggests the same expert, we construct the probabilities of each expert at time $t$ by normalization (to create a probability simplex), i.e.,
	\begin{align}
	p_{t,m}=\frac{w_{t,m}}{\sum_{m'}w_{t,m'}}.\label{pmt}
	\end{align}
	
	We emphasize that the strategies to be combined are not necessarily selected a priori. Instead, at each time $t$, all of the strategies $S_t$ that compromise the class $\mathbb{M}_t$ are treated as experts in our online learning problem \cite{singer,singer2, comp2} (which causes the universal property). These strategies are combined according to their weights $w_{S_t}$, indicating our trust in different strategies, to achieve the performance of any one of these strategies. Hence, our algorithm intrinsically achieves the performance of the optimal strategy without knowing which strategy specifically has the best performance because of its universal prediction perspective \cite{merhav}. 
	
	We point out that the construction of $p_{m,t}$ in \eqref{pmt} directly depends on $w_{S_t}$, the weight we assign to each strategy, in lieu of \eqref{wmt}, which we need to calculate at every round $t$. In a brute force approach, where we combine all possible expert selection strategies, the number of strategies combined grows exponentially and the computational cost becomes rapidly exhaustive. Thus, as we have mentioned at the beginning, this naive approach of treating each strategy as an expert and mixing them is not tractable. To solve this problem, we propose to mutually process and update -distinct but suitable to combine- strategy weights (as opposed to individually). Hence, because of the inefficiency of individual processing, instead of calculating each strategies' weight separately, we combine them into various equivalence classes for efficient an implementation.
	
	To create the equivalence classes, we first define an equivalence class parameter $\lambda_t$ as
	\begin{align}
		\lambda_t=[m, \ldots],\label{lamt}
	\end{align}
	where the first parameter $\lambda_t(1)$ is arbitrarily set as the expert selection $m$ at time $t$. Together with the omitted remaining parameters in \eqref{lamt}, $\lambda_t$ will determine the strategies that are included in that equivalence class, i.e., the equivalence class with parameters $\lambda_t$ includes all the strategies $S_t$ whose behavior match with the parameter vector $\lambda_t$ as a whole (e.g., they have to select the $\lambda_t(1)^{th}$ expert at time $t$). The parameters included in $\lambda_t$ determine its extend and how many different strategies it represents, which in turn determines how many equivalence classes we will have at the end for implementation. We define $\Omega_t$ as the vector space including all possible $\lambda_t$ vectors as
	\begin{align}
		\lambda_t\in\Omega_t, \enspace\forall\lambda_t.\label{omegat}
	\end{align}
	We point out that $\Omega_t$ may not necessarily represent all possible strategies at time $t$, but instead the strategies of our interest, which we want to compete against. We also define $\Lambda_t$ as the parameter sequence up to time $t$ for an arbitrary strategy as
	\begin{align}
		\Lambda_t\triangleq\{\lambda_1,\ldots,\lambda_t\},\label{Lamt}
	\end{align} 
	where each strategy $S_t$ will correspond to only one $\Lambda_t$. 
	 
	The reason for using auxiliary parameters $\lambda_t$ is to group together certain strategies with similar weight updates.	We define $w_{\lambda_t}$ as the weight of the equivalence class parameters $\lambda_t$ at time $t$. The weight of an equivalence class is simply the summation of the weights of the strategies whose behavior conforms with its class parameters $\lambda_t$, such that
	\begin{align}
	w_{\lambda_t}=\sum_{F_\lambda(S_t)=\lambda_t}^{}w_{S_t},\label{wlt}
	\end{align}
	where $F_\lambda(\cdot)$ is the mapping from strategies $S_t$ to the auxiliary parameters $\lambda_t$, i.e., {$F_\lambda:\Mt\rightarrow\Omega_t$}. Since we have discarded the strategy representation and individual weighting, the definition in \eqref{wmt} (which is each expert's weight) transforms to
	\begin{align}
	w_{t,m}=\sum_{\lambda_t(1)=m}^{} w_{\lambda_t}.\label{wmt2}
	\end{align}
	\begin{algorithm}[!t]
		\caption{Generalized Algorithm for Expert Selection}\label{alg:framework}
		\small{\begin{algorithmic}[1]
				\FOR{$t=1$ \TO $T$}
				\STATE Select $i_t\in\{1,\ldots,M\}$ with $p_t=[p_{t,1},\ldots,p_{t,M}]$
				\STATE Receive $\phi_{t}=[\phi_{t,1},\ldots,\phi_{t,M}]$
				\FOR{$\lambda_t\in\Omega_t$}
				\STATE $$z_{\lambda_t}=w_{\lambda_t}\exp(-\eta_{t-1}\phi_{t,\lambda_t(1)})$$
				\ENDFOR
				\FOR{$\lambda_{t+1}\in\Omega_{t+1}$}
				\STATE $$w_{\lambda_{t+1}}=\sum_{\lambda_{t}\in\Omega_t}\Tau(\lambda_{t+1}|\lambda_t)z_{\lambda_t}^{\frac{\eta_{t}}{\eta_{t-1}}}$$
				\ENDFOR
				\FOR{$m\in\{1,\ldots,M\}$}
				\STATE $$w_{t+1,m}={\sum_{\lambda_t(1)=m}w_{\lambda_{t+1}}}$$
				\ENDFOR
				\FOR{$m\in\{1,\ldots,M\}$}
				\STATE $$p_{t+1,m}=\frac{w_{t+1,m}}{\sum_{m=1}^{M}w_{t+1,m}}$$
				\ENDFOR
				\ENDFOR
		\end{algorithmic}}
	\end{algorithm}
	We update the weights $w_{\lambda_t}$ using the following two-step approach. In the first step, we define an intermediate variable $z_{\lambda_t}$ (which incorporates the exponential performance update as in the exponential weighting algorithm \cite{cesabook,cesa-bianchi,cesa2007}) as
	\begin{align}
	z_{\lambda_t}\triangleq w_{\lambda_t}e^{-\eta_{t-1}\phi_{t,\lambda_t(1)}},\label{zlt}
	\end{align}
	where $\phi_{t,m}$ is a measure of the experts performance (but not necessarily the loss $l_{t,m}$ itself), which we discuss more in the next section.
	In the second step, we create a probability sharing network among the equivalence classes (which also represents and assigns a weight to every individual strategy $S_t$ implicitly) at time $t$ as
	\begin{align}
	w_{\lambda_{t+1}}=\sum_{\lambda_t\in\Omega_t}\Tau(\lambda_{t+1}|\lambda_t)z_{\lambda_t}^{\frac{\eta_t}{\eta_{t-1}}},\label{wlt+}
	\end{align}
	where $\Tau(\lambda_{t+1}|\lambda_t)$ is the transition weight from the class parameters $\lambda_t$ to $\lambda_{t+1}$ such that $\sum_{\lambda_{t+1}\in\Omega_{t+1}}\Tau(\lambda_{t+1}|\lambda_t)=1$ (which is a probability simplex itself). The exponent on $z_{\lambda_t}$ (which is the intermediate variable defined in \eqref{zlt}) is for the iterative normalization of the learning rates $\eta_t$ and it is crucial for adaptive and parameter-free nature of our framework.  A summary of the method is given in Algorithm \ref{alg:framework}.
	
\section{Regret Analysis}\label{sec:regret}
	In this section, we prove the performance results of our algorithm. We first provide a summary of some important notations and definitions that will be heavily used in this section.
	\subsection{Notation and Definitions}\label{secsec:not}
	\begin{enumerate}
		\item $p_{t,m}$ is the probability of selecting $m$ at $t$ as in \eqref{pmt}.
		\item $\E_{p_{t,m}}[x_{t,m}]$ (or $\E_{p_{t,m}}x_{t,m}$ for brevity) is the expectation of $x_{t,m}$ over $p_{t,m}$, i.e., ${\E_{p_{t,m}}x_{t,m}=\sum_{m=1}^{M}p_{t,m}x_{t,m}}$.
		\item $\eta_t$ is the learning rate used in \eqref{zlt}.
		\item $\phi_{t,m}$ is the performance metric used in \eqref{zlt}.
		\item $d_t\triangleq \enspace\max_m\phi_{t,m}-\min_m\phi_{t,m}$ (i.e., range).
		\item $v_t\triangleq \enspace\E_{p_{t,m}}\phi_{t,m}^2$.
		\item $D_t\triangleq\max_{1\leq t' \leq t}d_t,$.
		\item $V_t\triangleq\sum_{t'=1}^t v_t$.
		\item $e$ is Euler's number.
		\item $\log(\cdot)$ is the natural logarithm. 
		\item $\lambda_t$ is an equivalence class parameter at time $t$ as in \eqref{lamt}.
		\item $\Omega_t$ is the set of all $\lambda_t$ at time $t$ as in \eqref{omegat}.
		\item $\Lambda_T\triangleq\{\lambda_t\}_{t=1}^T$ as in \eqref{Lamt}.
		\item $z_{\lambda_{t}}$ is as in \eqref{zlt}.
		\item $\Tau(\cdot|\cdot)$ is the transition weight used in \eqref{wlt+}.
		\item $\Tau(\{\lambda_t\}_{t=1}^T)\triangleq\prod_{t=1}^T\Tau(\lambda_t|\lambda_{t-1})$.		
		\item $W(\Lambda_T)\triangleq 1+\log(\max_{1\leq t\leq T}|\Omega_{t-1}|)-\log(\Tau(\Lambda_T))$.
\end{enumerate}
	
	\subsection{Useful Lemmas}
	To derive the regret bounds of our framework, we first determine a term of interest 
	\begin{align}
		\frac{1}{\eta_t}\log\E_{p_{t,m}}[e^{-\eta_t\phi_{t,m}}],\label{Eexp}
	\end{align}
	and use it to derive some useful Lemmas.
	\begin{lemma}\label{lem:lB}
		For any probability simplex $p_{t,m}$, we have the following inequality
		\begin{align}
		\frac{1}{\eta_t}\log\E_{p_{t,m}}[e^{-\eta_t\phi_{t,m}}]\leq-\E_{p_{t,m}}\phi_{t,m}+(e-2)\eta_t\E_{p_{t,m}}\phi_{t,m}^2,\nonumber
		\end{align}
		when $-\eta_t\phi_{t,m}\leq 1$, for all $t,m$.
		\begin{proof}
			The proof uses the inequality $e^x\leq1+x+(e-2)x^2$, when $x\leq1$ (which is comes from Taylor series \cite{handbook}).
		\end{proof}		
	\end{lemma}
	This Lemma puts an upper bound to our term of interest in \eqref{Eexp}. Similarly, we also have the following Lemma, which is a lower bound to that same term.
	\begin{lemma}\label{lem:uB1}
		For any probability simplex $p_{t,m}$, we have the following inequality
		\begin{align}
		\frac{1}{\eta_{t}}\log\E_{p_{t,m}}[e^{-\eta_{t}\phi_{t,m}}]\geq
		&\frac{1}{\eta_{t-1}}\log\E_{p_{t,m}}[e^{-\eta_{t-1}\phi_{t,m}}]\nonumber\\
		&-\left|1-\frac{\eta_t}{\eta_{t-1}}\right|d_t,\nonumber
		\end{align}
		where the operation $|\cdot|$ gives the absolute value. 
		\begin{proof}
			The proof is in Appendix \ref{pro:lem:uB1}.
		\end{proof}
	\end{lemma}
	Note that Lemma \ref{lem:uB1} provides only a partial bound. To further bound the term, we have the following Lemma.
	\begin{lemma}\label{lem:uB2}
		When using Algorithm \ref{alg:framework}, we have the following inequality
		\begin{align}
		\frac{1}{\eta_{t-1}}\log\E_{p_{t,m}}[e^{-\eta_{t-1}\phi_{t,m}}]\geq&\frac{1}{\eta_{t-1}}\log\left({\sum_{\lambda_t\in\Omega_t}z_{\lambda_t}}\right)\nonumber\\
		&-\frac{1}{\eta_{t-2}}\log\left({\sum_{\lambda_{t-1}\in\Omega_{t-1}}z_{\lambda_{t-1}}}\right)\nonumber\\
		&-(\frac{1}{\eta_{t-1}}-\frac{1}{\eta_{t-2}})\log(|\Omega_{t-1}|),\nonumber
		\end{align}
		when $\eta_t$ is nonincreasing with $t$.
	\end{lemma}
	\begin{proof}
		The proof is in Appendix \ref{pro:lem:uB2}.
	\end{proof}
	In Lemma \ref{lem:uB2}, we have succeeded in completing the bound for the individual terms (at time $t$). However, our goal is to bound their summation (from $t=1$ to $T$), which will require the following Lemma.
	\begin{lemma}\label{lem:uB3}
		When using Algorithm \ref{alg:framework}, we have
		\begin{align*}
		\frac{1}{\eta_{T-1}}\log(z_{\lambda_T})\geq-\sum_{t=1}^T\phi_{t,\lambda_{t}(1)}+\sum_{t=1}^T\frac{1}{\eta_{t-1}}\log(\Tau(\lambda_t|\lambda_{t-1})),
		\end{align*}
		for any sequence of equivalence classes $\{\lambda_{t}\}_{t=1}^T$ when $|\Omega_0|=1$.
		\begin{proof}
			The proof is in Appendix \ref{pro:lem:uB3}.
		\end{proof}
	\end{lemma}
	Now, we can combine Lemma \ref{lem:uB1}, \ref{lem:uB2} and \ref{lem:uB3} in the following to provide a lower bound to the summation of interest.
	\begin{lemma}\label{lem:uB4}
		When using Algorithm \ref{alg:framework}, we have
		\begin{align*}
		\sum_{t=1}^T\frac{1}{\eta_{t}}\log\left(\E_{p_{t,m}}[e^{-\eta_t\phi_{t,m}}]\right)\geq&-\sum_{t=1}^T\phi_{t,\lambda_t(1)}-\sum_{t=1}^{T}(1-\frac{\eta_t}{\eta_{t-1}})d_t\nonumber\\
		&+\sum_{t=1}^T\frac{1}{\eta_{t-1}}\log\left(\Tau(\lambda_{t}|\lambda_{t-1})\right)\nonumber\\
		&-\frac{1}{\eta_{T-1}}\log\left(\max_{1\leq t\leq T}|\Omega_{t-1}|\right),
		\end{align*}
		when $\eta_t$ is nonincreasing with $t$.
		\begin{proof}
			The proof is in Appendix \ref{pro:lem:uB4}.
		\end{proof}
	\end{lemma}
	With Lemma \ref{lem:uB4}, we now have a lower bound to our summation of interest, which is the summation of our term of interest in \eqref{Eexp} from $t=1$ to $T$. 
	
	\subsection{Performance Results}
	We combine Lemma \ref{lem:lB} and \ref{lem:uB4} together, which are upper and lower bounds to our summation of interest to we get the following Theorem.
	\begin{theorem}\label{thm:bound}
		When using Algorithm \ref{alg:framework}, we have
		\begin{align*}
		\sum_{t=1}^T\left(\E_{p_{t,m}}\phi_{t,m}-\phi_{t,\lambda_t(1)}\right)\leq& (e-2)\sum_{t=1}^T\eta_t\E_{p_{t,m}}\phi_{t,m}^2\\
		&+\frac{\log(\max_{1\leq t\leq T}|\Omega_{t-1}|)}{\eta_{T-1}}\\
		&-\frac{1}{\eta_{T-1}}\log(\Tau(\Lambda_T))\\
		&+\sum_{t=1}^T(1-\frac{\eta_t}{\eta_{t-1}})d_t,
		\end{align*}
		where $\Tau(\Lambda_T)=\Tau(\{\lambda_t\}_{t=1}^T)$; $-\eta_t\phi_{t,m}\leq 1$, for all $t,m$; $\eta_t$ is nonincreasing with $t$.
		\begin{proof}
			The proof is in Appendix \ref{pro:thm:bound}. 
		\end{proof}
	\end{theorem}
	Theorem \ref{thm:bound} provides us an upper bound on the expected cumulative difference on the performance variable $\phi_{t,m}$ (possibly 'regret' itself which will be explained in the remainder of the section) in terms of the learning rates $\eta_t$. The selection of the learning rates drastically affects the upper bound and should be chosen with care. To this end, we set the learning rates as the following
	\begin{align}
		\eta_t=\frac{\gamma}{\sqrt{V_t+\gamma^2D_t^2}},\label{etat}
	\end{align}
	where $\gamma$ is a user-set parameter.
	\begin{remark}\label{rem:1}
		When $\eta_{t}$ is chosen as \eqref{etat}, we have $\eta_{t}\leq\eta_{t-1}$ for all $t$ and $-\eta_{t}\phi_{t,m}\leq1$ for all $t$ and $m$ if for every $t$ there is at least one $m'$ such that $\phi_{t,m'}\geq0$ (which will be deliberated on in the remainder of this section), which satisfies our requirements in Lemma \ref{lem:uB2}, \ref{lem:uB4} and Theorem \ref{thm:bound}.
	\end{remark}
	\begin{theorem}\label{thm:bound2}
		When $\eta_t=\frac{\gamma}{\sqrt{V_t+\gamma^2D_t^2}}$ in Algorithm \ref{alg:framework}, we have
		\begin{align*}
		\sum_{t=1}^T(\E_{p_{t,m}}\phi_{t,m}-\phi_{t,m_t})
		\leq&\frac{W(\Lambda_T)}{\gamma}\sqrt{V_T+\gamma^2D_T^2}\nonumber\\
		&+{2(e-2)\gamma\sqrt{V_T}},
		\end{align*}
		where $\gamma$ is a user-set parameter.
		\begin{proof}
			The proof is in Appendix \ref{pro:thm:bound2}.
		\end{proof}
	\end{theorem}
	Theorem \ref{thm:bound2} provides us with a performance bound that is only dependent on a single parameter $\gamma$ which needs to be set at the beginning. However, this does not invalidate the truly online claim since $\gamma$ can be straightforwardly set based on the size of the competition class alone, which is something we naturally have access to at the design of the algorithm. 
	\begin{corollary}\label{cor:1}
		When $\gamma=\sqrt{\frac{W_T}{{2(e-2)}}}$, where $W_T$ is an upper bound on our competing class such that $W(\Lambda_T)\leq W_T$, we have
		\begin{align*}
		\sum_{t=1}^T(\E_{p_{t,m}}\phi_{t,m}-\phi_{t,m_t})\leq&W_TD_T+{2.4\sqrt{W_TV_T}},
		\end{align*}
		\begin{proof}
			The proof is in the Appendix \ref{pro:cor:1}.
		\end{proof}
	\end{corollary}
	
	\begin{remark}
		For any $t$, let $\phi_{t,m}=l_{t,m}-\mu_t$ for all $m$ for some $\mu_t$. Then,
		\begin{align}
			\sum_{t=1}^{T}(\E_{p_{t,m}}\phi_{t,m}-\phi_{t,m_t})=\sum_{t=1}^{T}(\E_{p_{t,m}}l_{t,m}-l_{t,m_t}),
		\end{align}
		hence, all performance bounds in this section will hold as regret bounds as long as $\phi_{t,m}$ is a translation of $l_{t,m}$.
	\end{remark}
	\begin{remark}
		For every $\{\mu_t\}_{t=1}^T\in \Re^T$, $\{d_t\}_{t=1}^T$ and $\{D_t\}_{t=1}^T$ remain unchanged, since for all $t$ 
		\begin{align}
			d_t=\max_m l_{t,m} - \min_m l_{t,m}.
		\end{align}
	\end{remark}
	\begin{remark}\label{rem:2}
		The sequence of $\{\mu_t \}_{t=1}^T$ that minimizes the regret bounds is
		\begin{align}
			\mu_t^*=\E_{p_{t,m}}l_{t,m},
		\end{align}
		since 
		\begin{align}
			V_T=&\sum_{t=1}^{T}v_t\\
			=&\sum_{t=1}^{T}\E_{p_{t,m}}(l_{t,m}-\mu_t)^2\\
			=&\sum_{t=1}^{T}\E_{p_{t,m}}(l_{t,m}-\E_{p_{t,m}}l_{t,m})^2+(\E_{p_{t,m}}l_{t,m}-\mu_t)^2,
		\end{align}
		which also satisfies our requirement in Remark \ref{rem:1} that there is at least one $m'$ such that $\phi_{t,m'}\geq0$ for every $t$ individually .
	\end{remark}
	Since the weights are updated after the declaration of $p_{t,m}$ and observation of $l_{t,m}$, there is no problem in using the translation of Remark \ref{rem:2}. Hence, without issue, we can set the performance metric as
	\begin{align}
		\phi_{t,m}=l_{t,m}-\E_{p_{t,m}}l_{t,m}.\label{phitm}
	\end{align} 
	\begin{corollary}\label{cor:1.1}
		When $\phi_{t,m}$ is set as \eqref{phitm}, our result in Corollary \ref{cor:1} becomes 
		\begin{align*}
			\sum_{t=1}^T(\E_{p_{t,m}}l_{t,m}-l_{t,m_t})\leq&W_TD_T+{2.4\sqrt{W_TV_T^*}},
		\end{align*}
		such that $V_T^*$ is the sum of loss variances with our selection probabilities.
		\begin{proof}
			The proof is straightforward by using \eqref{phitm} in Corollary \ref{cor:1}.
		\end{proof}
	\end{corollary}
	We can also straightforwardly acquire the following regret bound that is not dependent on the selection probabilities $p_{t,m}$.
	\begin{corollary}\label{cor:1.2}
		We also have the following result instead of the one in Corollary \ref{cor:1.1}, which is
		\begin{align*}
		\sum_{t=1}^T(\E_{p_{t,m}}l_{t,m}-l_{t,m_t})\leq& W_TD_T+1.2{\sqrt{W_T\sum_{t=1}^{T}d_t^2}}.
		\end{align*}
		\begin{proof}
			The proof uses the fact that the variance of a loss with respect to any probability simplex is at most $d_t^2/4$.
		\end{proof}
	\end{corollary}
	
	\section{Example Application}\label{sec:example}
	After providing the algorithmic framework in Section \ref{sec:method} and the accompanying regret bounds in Section \ref{sec:regret}, in this section, we finally demonstrate the construction of our algorithm for a specific problem and its competition class. 
	
	As an example, we consider the problem of competing against evolving expert selection strategies. In the following toy example, we consider the strategies with a moving rate $\sigma$ such that if a strategy choses the expert $m\in\{0,1,\ldots,M-1\}$ at round $t$, then it will select the expert $m'=(m+\sigma)(\bmod\enspace M)$ at time $t+1$. 
	
	We aim to compete against all such strategies with fixed moving rate $\sigma$ which may not be necessarily bounded from above. However, in this problem, since the expert transitions between rounds follows a cyclic behavior (because of the mod operation), we, in truth, only have $M$ such unique moving rates which are $\sigma\in\{0,1,\ldots,M-1\}$. 
	
	We can utilize our framework by designing the equivalence class with parameters $\lambda_t$ at time $t$ which includes the expert selection $m$ at time $t$ and the moving rate $\sigma$ as
	\begin{align}
		\lambda_t=[m,\sigma].
	\end{align}
	Thus, we have, in total, $M^2$ equivalence classes. The algorithm becomes as the following:
	$\Tau(\lambda_{t+1}|\lambda_t)=1$ if $\lambda_{t+1}(2)=\lambda_{t}(2)$ and $\lambda_{t+1}(1)=(\lambda_t(1)+\lambda_t(2))(\bmod\enspace M)$; and $0$ otherwise. Hence, we have
	\begin{align}
		\Tau([(m+\sigma)(\bmod M),\sigma]|[m,\sigma])=1.
	\end{align}
	Since a strategy from the competition class has $\Tau(\Lambda_T)=1$, we have $W(\lambda_T)=1+2\log(M)=W_T$. Thus, we get the following result.
	\begin{corollary}\label{cor:2}
		When $\gamma=\sqrt{\frac{1+2\log(M)}{{2(e-2)}}}$, we have
		\begin{align*}
		\sum_{t=1}^T(\E_{p_{t,m}}l_{t,m}-l_{t,m_t})\leq&(1+2\log(M))D_T\\
		&+{2.4\sqrt{(1+2\log(M))V_T^*}},
		\end{align*}
		where $D_T$ is the maximum loss range in $T$ rounds, and $V_T$ is the sum of loss variances.
		\begin{proof}
			The proof is straightforward by application of Corollary \ref{cor:1.1} with $W_T=1+2\log(M)$.
		\end{proof}
	\end{corollary}

	\section{Conclusion}\label{sec:conc}
	In conclusion, we have successfully created a completely online, generalized algorithm for prediction by expert advice. Our performance bounds are translation-free and scale-free of the expert losses. With suitable design, it is possible to compete against a subset of the all possible expert selection strategies that is appropriate for a given problem. By combining the similar strategies together in each step of the algorithm, and creating appropriate equivalence classes, we can compete against the strategies with minimal redundancy and in a computationally efficient manner.
	
\bibliographystyle{ieeetran}
\bibliography{double_bib}	

% Generated by IEEEtran.bst, version: 1.14 (2015/08/26)
\begin{thebibliography}{10}
\providecommand{\url}[1]{#1}
\csname url@samestyle\endcsname
\providecommand{\newblock}{\relax}
\providecommand{\bibinfo}[2]{#2}
\providecommand{\BIBentrySTDinterwordspacing}{\spaceskip=0pt\relax}
\providecommand{\BIBentryALTinterwordstretchfactor}{4}
\providecommand{\BIBentryALTinterwordspacing}{\spaceskip=\fontdimen2\font plus
\BIBentryALTinterwordstretchfactor\fontdimen3\font minus
  \fontdimen4\font\relax}
\providecommand{\BIBforeignlanguage}[2]{{%
\expandafter\ifx\csname l@#1\endcsname\relax
\typeout{** WARNING: IEEEtran.bst: No hyphenation pattern has been}%
\typeout{** loaded for the language `#1'. Using the pattern for}%
\typeout{** the default language instead.}%
\else
\language=\csname l@#1\endcsname
\fi
#2}}
\providecommand{\BIBdecl}{\relax}
\BIBdecl

\bibitem{cesabook}
N.~Cesa-Bianchi and G.~Lugosi, \emph{Prediction, learning, and games}.\hskip
  1em plus 0.5em minus 0.4em\relax Cambridge university press, 2006.

\bibitem{chang}
H.~S. Chang, J.~Hu, M.~C. Fu, and S.~I. Marcus, ``Adaptive adversarial
  multi-armed bandit approach to two-person zero-sum markov games,'' \emph{IEEE
  Transactions on Automatic Control}, vol.~55, no.~2, pp. 463--468, Feb 2010.

\bibitem{tnnls1}
R.~Song, F.~L. Lewis, and Q.~Wei, ``Off-policy integral reinforcement learning
  method to solve nonlinear continuous-time multiplayer nonzero-sum games,''
  \emph{IEEE Transactions on Neural Networks and Learning Systems}, vol.~PP,
  no.~99, pp. 1--10, 2016.

\bibitem{sw2}
A.~Heydari and S.~N. Balakrishnan, ``Optimal switching and control of nonlinear
  switching systems using approximate dynamic programming,'' \emph{IEEE
  Transactions on Neural Networks and Learning Systems}, vol.~25, no.~6, pp.
  1106--1117, June 2014.

\bibitem{sw4}
X.~Liu, H.~Su, and M.~Z.~Q. Chen, ``A switching approach to designing
  finite-time synchronization controllers of coupled neural networks,''
  \emph{IEEE Transactions on Neural Networks and Learning Systems}, vol.~27,
  no.~2, pp. 471--482, Feb 2016.

\bibitem{tnnls3}
H.~R. Berenji and P.~Khedkar, ``Learning and tuning fuzzy logic controllers
  through reinforcements,'' \emph{IEEE Transactions on Neural Networks},
  vol.~3, no.~5, pp. 724--740, Sep 1992.

\bibitem{tnnls4}
J.~Moody and M.~Saffell, ``Learning to trade via direct reinforcement,''
  \emph{IEEE Transactions on Neural Networks}, vol.~12, no.~4, pp. 875--889,
  Jul 2001.

\bibitem{comp1}
P.~Auer and M.~K. Warmuth, ``Tracking the best disjunction,'' \emph{Machine
  Learning}, vol.~32, no.~2, pp. 127--150, 1998.

\bibitem{comp2}
M.~Herbster and M.~K. Warmuth, ``Tracking the best expert,'' \emph{Machine
  Learning}, vol.~32, no.~2, pp. 151--178, 1998.

\bibitem{merhav}
N.~Merhav and M.~Feder, ``Universal prediction,'' \emph{IEEE Transactions on
  Information Theory}, vol.~44, no.~6, pp. 2124--2147, 1998.

\bibitem{gIncremental}
K.~{Gokcesu}, M.~M. {Neyshabouri}, H.~{Gokcesu}, and S.~S. {Kozat},
  ``Sequential outlier detection based on incremental decision trees,''
  \emph{IEEE Transactions on Signal Processing}, vol.~67, no.~4, pp. 993--1005,
  2019.

\bibitem{sw3}
P.~Lim, C.~K. Goh, K.~C. Tan, and P.~Dutta, ``Multimodal degradation
  prognostics based on switching kalman filter ensemble,'' \emph{IEEE
  Transactions on Neural Networks and Learning Systems}, vol.~PP, no.~99, pp.
  1--13, 2016.

\bibitem{signal2}
A.~J. Bean and A.~C. Singer, ``Universal switching and side information
  portfolios under transaction costs using factor graphs,'' \emph{IEEE Journal
  of Selected Topics in Signal Processing}, vol.~6, no.~4, pp. 351--365, Aug
  2012.

\bibitem{moon}
T.~Moon and T.~Weissman, ``Universal fir mmse filtering,'' \emph{IEEE
  Transactions on Signal Processing}, vol.~57, no.~3, pp. 1068--1083, March
  2009.

\bibitem{signal1}
T.~Moon, ``Universal switching fir filtering,'' \emph{IEEE Transactions on
  Signal Processing}, vol.~60, no.~3, pp. 1460--1464, March 2012.

\bibitem{sw5}
A.~Heydari, ``Feedback solution to optimal switching problems with switching
  cost,'' \emph{IEEE Transactions on Neural Networks and Learning Systems},
  vol.~PP, no.~99, pp. 1--1, 2015.

\bibitem{gHierarchical}
N.~D. {Vanli}, K.~{Gokcesu}, M.~O. {Sayin}, H.~{Yildiz}, and S.~S. {Kozat},
  ``Sequential prediction over hierarchical structures,'' \emph{IEEE
  Transactions on Signal Processing}, vol.~64, no.~23, pp. 6284--6298, 2016.

\bibitem{ozkan}
H.~Ozkan, M.~A. Donmez, S.~Tunc, and S.~S. Kozat, ``A deterministic analysis of
  an online convex mixture of experts algorithm,'' \emph{IEEE Transactions on
  Neural Networks and Learning Systems}, vol.~26, no.~7, pp. 1575--1580, July
  2015.

\bibitem{singer}
A.~C. Singer and M.~Feder, ``Universal linear prediction by model order
  weighting,'' \emph{IEEE Transactions on Signal Processing}, vol.~47, no.~10,
  pp. 2685--2699, Oct 1999.

\bibitem{singer2}
------, ``Universal linear least-squares prediction,'' in \emph{Information
  Theory, 2000. Proceedings. IEEE International Symposium on}, 2000, pp.
  81--81.

\bibitem{gAnomaly}
K.~{Gokcesu} and S.~S. {Kozat}, ``Online anomaly detection with minimax optimal
  density estimation in nonstationary environments,'' \emph{IEEE Transactions
  on Signal Processing}, vol.~66, no.~5, pp. 1213--1227, 2018.

\bibitem{willems}
F.~M.~J. Willems, ``Coding for a binary independent
  piecewise-identically-distributed source.'' \emph{IEEE Transactions on
  Information Theory}, vol.~42, no.~6, pp. 2210--2217, 1996.

\bibitem{coding1}
N.~Merhav, ``On the minimum description length principle for sources with
  piecewise constant parameters,'' \emph{IEEE Transactions on Information
  Theory}, vol.~39, no.~6, pp. 1962--1967, Nov 1993.

\bibitem{gDensity}
K.~{Gokcesu} and S.~S. {Kozat}, ``Online density estimation of nonstationary
  sources using exponential family of distributions,'' \emph{IEEE Transactions
  on Neural Networks and Learning Systems}, vol.~29, no.~9, pp. 4473--4478,
  2018.

\bibitem{coding2}
G.~I. Shamir and N.~Merhav, ``Low-complexity sequential lossless coding for
  piecewise-stationary memoryless sources,'' \emph{IEEE Transactions on
  Information Theory}, vol.~45, no.~5, pp. 1498--1519, Jul 1999.

\bibitem{sw1}
X.~Liu, J.~Lam, W.~Yu, and G.~Chen, ``Finite-time consensus of multiagent
  systems with a switching protocol,'' \emph{IEEE Transactions on Neural
  Networks and Learning Systems}, vol.~27, no.~4, pp. 853--862, April 2016.

\bibitem{vanli}
N.~D. Vanli, M.~O. Sayin, I.~Delibalta, and S.~S. Kozat, ``Sequential nonlinear
  learning for distributed multiagent systems via extreme learning machines,''
  \emph{IEEE Transactions on Neural Networks and Learning Systems}, vol.~PP,
  no.~99, pp. 1--13, 2016.

\bibitem{tekin1}
C.~Tekin, S.~Zhang, and M.~van~der Schaar, ``Distributed online learning in
  social recommender systems,'' \emph{IEEE Journal of Selected Topics in Signal
  Processing}, vol.~8, no.~4, pp. 638--652, 2014.

\bibitem{auerExp}
P.~Auer, N.~Cesa-Bianchi, Y.~Freund, and R.~E. Schapire, ``The nonstochastic
  multiarmed bandit problem,'' \emph{SIAM J. Comput.}, vol.~32, no.~1, pp.
  48--77, Jan. 2003.

\bibitem{bandit1}
J.~C. Gittins, ``Bandit processes and dynamic allocation indices,''
  \emph{Journal of the Royal Statistical Society. Series B (Methodological)},
  vol.~41, no.~2, pp. 148--177, 1979.

\bibitem{bandit2}
P.~Auer, N.~Cesa-Bianchi, Y.~Freund, and R.~E. Schapire, ``Gambling in a rigged
  casino: The adversarial multi-armed bandit problem,'' in \emph{Foundations of
  Computer Science, 1995. Proceedings., 36th Annual Symposium on}, Oct 1995,
  pp. 322--331.

\bibitem{audibert}
J.-Y. Audibert and S.~Bubeck, ``Regret bounds and minimax policies under
  partial monitoring,'' \emph{J. Mach. Learn. Res.}, vol.~11, pp. 2785--2836,
  Dec. 2010.

\bibitem{tekin2}
C.~Tekin and M.~van~der Schaar, ``Releaf: An algorithm for learning and
  exploiting relevance,'' \emph{IEEE Journal of Selected Topics in Signal
  Processing}, vol.~9, no.~4, pp. 716--727, 2015.

\bibitem{gBandit}
K.~{Gokcesu} and S.~S. {Kozat}, ``An online minimax optimal algorithm for
  adversarial multiarmed bandit problem,'' \emph{IEEE Transactions on Neural
  Networks and Learning Systems}, vol.~29, no.~11, pp. 5565--5580, 2018.

\bibitem{exptrade}
L.~P. Kaelbling, M.~L. Littman, and A.~W. Moore, ``Reinforcement learning: A
  survey,'' \emph{Journal of artificial intelligence research}, vol.~4, pp.
  237--285, 1996.

\bibitem{auer}
P.~Auer, ``Using confidence bounds for exploitation-exploration trade-offs,''
  \emph{J. Mach. Learn. Res.}, vol.~3, pp. 397--422, Mar. 2003.

\bibitem{auerSelf}
\BIBentryALTinterwordspacing
P.~Auer, N.~Cesa-Bianchi, and C.~Gentile, ``Adaptive and self-confident on-line
  learning algorithms,'' \emph{Journal of Computer and System Sciences},
  vol.~64, no.~1, pp. 48 -- 75, 2002. [Online]. Available:
  \url{http://www.sciencedirect.com/science/article/pii/S0022000001917957}
\BIBentrySTDinterwordspacing

\bibitem{littlestone1994}
\BIBentryALTinterwordspacing
N.~Littlestone and M.~Warmuth, ``The weighted majority algorithm,''
  \emph{Information and Computation}, vol. 108, no.~2, pp. 212 -- 261, 1994.
  [Online]. Available:
  \url{http://www.sciencedirect.com/science/article/pii/S0890540184710091}
\BIBentrySTDinterwordspacing

\bibitem{vovk1998}
\BIBentryALTinterwordspacing
V.~Vovk, ``A game of prediction with expert advice,'' \emph{Journal of Computer
  and System Sciences}, vol.~56, no.~2, pp. 153 -- 173, 1998. [Online].
  Available:
  \url{http://www.sciencedirect.com/science/article/pii/S0022000097915567}
\BIBentrySTDinterwordspacing

\bibitem{freund1997}
\BIBentryALTinterwordspacing
Y.~Freund and R.~E. Schapire, ``A decision-theoretic generalization of on-line
  learning and an application to boosting,'' \emph{Journal of Computer and
  System Sciences}, vol.~55, no.~1, pp. 119 -- 139, 1997. [Online]. Available:
  \url{http://www.sciencedirect.com/science/article/pii/S002200009791504X}
\BIBentrySTDinterwordspacing

\bibitem{cesa2007}
N.~Cesa-Bianchi, Y.~Mansour, and G.~Stoltz, ``Improved second-order bounds for
  prediction with expert advice,'' \emph{Machine Learning}, vol.~66, no. 2-3,
  pp. 321--352, 2007.

\bibitem{allenberg2004}
C.~Allenberg-Neeman and B.~Neeman, ``Full information game with gains and
  losses,'' in \emph{Algorithmic Learning Theory}, S.~Ben-David, J.~Case, and
  A.~Maruoka, Eds.\hskip 1em plus 0.5em minus 0.4em\relax Berlin, Heidelberg:
  Springer Berlin Heidelberg, 2004, pp. 264--278.

\bibitem{cesa-bianchi}
S.~Bubeck and N.~Cesa{-}Bianchi, ``Regret analysis of stochastic and
  nonstochastic multi-armed bandit problems,'' \emph{Foundations and Trends in
  Machine Learning}, vol.~5, no.~1, pp. 1--122, 2012.

\bibitem{handbook}
M.~Abramowitz and I.~A. Stegun, \emph{Handbook of mathematical functions with
  formulas, graphs, and mathematical tables}.\hskip 1em plus 0.5em minus
  0.4em\relax US Government printing office, 1948, vol.~55.

\end{thebibliography}
	
	\begin{appendices}
	
	\section{Proof of Lemma \ref{lem:uB1}}\label{pro:lem:uB1}
	First of all, we have
	\begin{align}
	\frac{1}{\eta_{t-1}}\log&\E_{p_{t,m}}[e^{-\eta_{t-1}\phi_{t,m}}]\nonumber\\	&=\frac{1}{\eta_{t-1}}\log\E_{p_{t,m}}[e^{-\eta_{t}\phi_{t,m}+(\eta_t-\eta_{t-1})\phi_{t,m}}]\nonumber\\
	&\leq\frac{1}{\eta_{t-1}}\log\E_{p_{t,m}}[e^{-\eta_{t}\phi_{t,m}+(\eta_t-\eta_{t-1})a_t}]\nonumber\\
	&\leq\frac{1}{\eta_{t-1}}\log\E_{p_{t,m}}[e^{-\eta_{t}\phi_{t,m}}]+\left(\frac{\eta_t}{\eta_{t-1}}-1\right)a_t,\label{eq:2}
	\end{align}
	where $a_t$ is either minimum or maximum of $\phi_{t,m}$ over $m$ depending on whether or not $\eta_{t-1}$ is greater than $\eta_{t}$, i.e.,
	\begin{align}
	a_t=\left\{
	\begin{array}{ll}
	\min_m\phi_{t,m} & \eta_t\leq\eta_{t-1}\\
	\max_m\phi_{t,m} & \eta_t\geq\eta_{t-1}
	\end{array}.
	\right.
	\end{align}
	Secondly, we also have
	\begin{align}
	\frac{1}{\eta_{t-1}}\log\E_{p_{t,m}}[e^{-\eta_{t}\phi_{t,m}}]-&\frac{1}{\eta_{t}}\log\E_{p_{t,m}}[e^{-\eta_{t}\phi_{t,m}}]\nonumber\\
	&\leq\left(\frac{1}{\eta_{t-1}}-\frac{1}{\eta_{t}}\right)\log\E_{p_{t,m}}[e^{-\eta_{t}b_t}]\nonumber\\
	&\leq-\left(\frac{\eta_t}{\eta_{t-1}}-1\right)b_t,\label{eq:3}
	\end{align}
	where $b_t$ is either minimum or maximum of $\phi_{t,m}$ over $m$ depending on whether or not $\eta_{t}$ is greater than $\eta_{t-1}$, i.e.,
	\begin{align}
	b_t=\left\{
	\begin{array}{ll}
	\max_m\phi_{t,m} & \eta_t\leq\eta_{t-1}\\
	\min_m\phi_{t,m} & \eta_t\geq\eta_{t-1}
	\end{array}.
	\right.
	\end{align}
	Combining \eqref{eq:2} and \eqref{eq:3}, we get
	\begin{align}
	\frac{1}{\eta_{t-1}}\log\E_{p_{t,m}}[e^{-\eta_{t-1}\phi_{t,m}}]&-\frac{1}{\eta_{t}}\log\E_{p_{t,m}}[e^{-\eta_{t}\phi_{t,m}}]\nonumber\\
	&\leq\left(\frac{\eta_t}{\eta_{t-1}}-1\right)(a_t-b_t)\nonumber\\
	&\leq\left|\frac{\eta_t}{\eta_{t-1}}-1\right|d_t,\label{eq:4}
	\end{align}
	where $d_t\triangleq(\max_m\phi_{t,m}-\min_m\phi_{t,m})$. Moreover, since
	\begin{align}
	-\frac{1}{\eta_t}\log\E_{p_{t,m}}[e^{-\eta_t\phi_{t,m}}]=&-\frac{1}{\eta_{t-1}}\log\E_{p_{t,m}}[e^{-\eta_{t-1}\phi_{t,m}}]\nonumber\\
	&+\frac{1}{\eta_{t-1}}\log\E_{p_{t,m}}[e^{-\eta_{t-1}\phi_{t,m}}]\nonumber\\
	&-\frac{1}{\eta_t}\log\E_{p_{t,m}}[e^{-\eta_t\phi_{t,m}}]\label{eq:1},
	\end{align}
	putting \eqref{eq:4} into \eqref{eq:1} concludes the proof.
	
	\section{Proof of Lemma \ref{lem:uB2}}\label{pro:lem:uB2}
	We continue from Lemma \ref{lem:uB1} to bound the excess term on the right hand side. To begin with, we have
	\begin{align}
	-\frac{1}{\eta_{t-1}}\log&\left(\E_{p_{t,m}}[e^{-\eta_{t-1}\phi_{t,m}}]\right)\nonumber\\
	&=-\frac{1}{\eta_{t-1}}\log\left(\frac{\sum_{m}w_{t,m}e^{-\eta_{t-1}\phi_{t,m}}}{\sum_{m'}w_{t,m'}}\right)\label{eq:pr1}\\
	&=-\frac{1}{\eta_{t-1}}\log\left(\frac{\sum_{\lambda_t\in\Omega_t}w_{\lambda_t}e^{-\eta_{t-1}\phi_{t,\lambda_t(1)}}}{\sum_{\lambda_t\in\Omega_t}w_{\lambda_t}}\right)\label{eq:pr2}\\
	&=-\frac{1}{\eta_{t-1}}\log\left(\frac{\sum_{\lambda_t\in\Omega_t}z_{\lambda_t}}{\sum_{\lambda_t\in\Omega_t}w_{\lambda_t}}\right),\label{eq:pr3}\\
	&=-\frac{1}{\eta_{t-1}}\log\left(\frac{\sum_{\lambda_t\in\Omega_t}z_{\lambda_t}}{\sum_{\lambda_{t-1}\in\Omega_{t-1}}z_{\lambda_{t-1}}^{\frac{\eta_{t-1}}{\eta_{t-2}}}}\right),\label{eq:5}
	\end{align}
	where \eqref{eq:pr1}, \eqref{eq:pr2}, \eqref{eq:pr3} and \eqref{eq:5} use results from \eqref{pmt}, \eqref{wlt}, \eqref{zlt} and \eqref{wlt+} respectively.
	Moreover, for the denominator in the logarithm, we have
	\begin{align}
	&\frac{1}{\eta_{t-1}}\log\left({\sum_{\lambda_{t-1}\in\Omega_{t-1}}z_{\lambda_{t-1}}^{\frac{\eta_{t-1}}{\eta_{t-2}}}}\right)\nonumber\\
	&\leq\frac{1}{\eta_{t-1}}\log\left({\sum_{\lambda_{t-1}}\frac{1}{|\Omega_{t-1}|}z_{\lambda_{t-1}}^{\frac{\eta_{t-1}}{\eta_{t-2}}}}\right)+\frac{\log(|\Omega_{t-1}|)}{\eta_{t-1}},\nonumber\\
	&\leq\frac{1}{\eta_{t-2}}\frac{\eta_{t-2}}{\eta_{t-1}}\log\left({\sum_{\lambda_{t-1}}\frac{1}{|\Omega_{t-1}|}z_{\lambda_{t-1}}^{\frac{\eta_{t-1}}{\eta_{t-2}}}}\right)+\frac{\log(|\Omega_{t-1}|)}{\eta_{t-1}},\nonumber\\
	&\leq\frac{1}{\eta_{t-2}}\log\left({\sum_{\lambda_{t-1}}\frac{1}{|\Omega_{t-1}|}z_{\lambda_{t-1}}}\right)+\frac{\log(|\Omega_{t-1}|)}{\eta_{t-1}},\label{eq:pr6.1}\\
	&\leq\frac{1}{\eta_{t-2}}\log\left({\sum_{\lambda_{t-1}}z_{\lambda_{t-1}}}\right)+(\frac{1}{\eta_{t-1}}-\frac{1}{\eta_{t-2}})\log(|\Omega_{t-1}|),\label{eq:6}
	\end{align}
	where \eqref{eq:pr6.1} uses Jensen's Inequality and the set $\Omega_{t-1}$ is omitted over the summations after the first line for space considerations.
	Putting \eqref{eq:6} into \eqref{eq:5}, we get
	\begin{align}
	-\frac{1}{\eta_{t-1}}\log\E_{p_{t,m}}[e^{-\eta_{t-1}\phi_{t,m}}]\leq&-\frac{1}{\eta_{t-1}}\log\left({\sum_{\lambda_t\in\Omega_t}z_{\lambda_t}}\right)\nonumber\\
	&+\frac{1}{\eta_{t-2}}\log\left({\sum_{\lambda_{t-1}\in\Omega_{t-1}}z_{\lambda_{t-1}}}\right)\nonumber\\
	&+(\frac{1}{\eta_{t-1}}-\frac{1}{\eta_{t-2}})\log(|\Omega_{t-1}|),\label{eq:7}
	\end{align}
	which concludes the proof.

	\section{Proof of Lemma \ref{lem:uB3}}\label{pro:lem:uB3}
	By definition in \eqref{zlt}, we have
	\begin{align}
	-\log(z_{\lambda_t})=\eta_{t-1}\phi_{t,\lambda_{t}(1)}-\log(w_{\lambda_{t}}),\label{eq:logzlt}
	\end{align}
	and from \eqref{wlt+}, we have	
	\begin{align}
	-\log(w_{t,\lambda_t})\leq-\log(\Tau(\lambda_{t}|\lambda_{t-1}))-\frac{\eta_{t-1}}{\eta_{t-2}}\log(z_{\lambda_{t-1}}).\label{eq:logwlt+}
	\end{align}
	Combining \eqref{eq:logzlt} and \eqref{eq:logwlt+}, we get
	\begin{align}
	-\frac{1}{\eta_{t-1}}\log(z_{\lambda_t})=&\enspace\phi_{t,\lambda_{t}(1)}-\frac{1}{\eta_{t-1}}\log w_{\lambda_t}\nonumber\\
	\leq&\enspace\phi_{t,\lambda_{t}(1)}-\frac{1}{\eta_{t-1}}\log(\Tau(\lambda_t|\lambda_{t-1}))\nonumber\\
	&-\frac{1}{\eta_{t-2}}\log(z_{\lambda_{t-1}}).\label{eq:logzlt+}
	\end{align}
	From the telescoping relation in \eqref{eq:logzlt+}, we get
	\begin{align}
	-\frac{1}{\eta_{T-1}}\log(z_{\lambda_T})\leq&\sum_{t=1}^T\phi_{t,\lambda_{t}(1)}-\sum_{t=1}^T\frac{1}{\eta_{t-1}}\log(\Tau(\lambda_t|\lambda_{t-1})),
	\end{align}
	since $z_{0,m}=1$ and concludes the proof.
	
	\section{Proof of Lemma \ref{lem:uB4}}\label{pro:lem:uB4}
	We sum \eqref{eq:7} from $t=1$ to $T$, and get
	\begin{align}
	\sum_{t=1}^{T}\frac{1}{\eta_{t-1}}&\log\E_{p_{t,m}}[e^{-\eta_{t-1}\phi_{t,m}}]\nonumber\\
	\geq&\enspace\frac{1}{\eta_{T-1}}\log(\sum_{\lambda_{T}\in\Omega_{T}}z_{\lambda_{T}})-\frac{1}{\eta_{-1}}\log(\sum_{\lambda_{0}\in\Omega_{0}}z_{\lambda_{0}})\nonumber\\
	&-\sum_{t=1}^{T}(\frac{1}{\eta_{t-1}}-\frac{1}{\eta_{t-2}})\log(|\Omega_{t-1}|),
	\end{align}
	where $\eta_{-1}$ and $\eta_{0}$ can be arbitrarily chosen as $\eta_{1}$ and $|\lambda_0|$ as $1$, $z_{\lambda_{0}}=1$.
	Then, using Lemma \ref{lem:uB1}, \ref{lem:uB2} and \ref{lem:uB3}, we get
	\begin{align}
	\sum_{t=1}^T\frac{1}{\eta_{t}}\log\left(\E_{p_{t,m}}[e^{-\eta_t\phi_{\lambda_t}}]\right)\geq&-\sum_{t=1}^T\phi_{t,\lambda_t(1)}-\sum_{t=1}^{T}(1-\frac{\eta_t}{\eta_{t-1}})d_t\nonumber\\
	&+\sum_{t=1}^T\frac{1}{\eta_{t-1}}\log\left(\Tau(\lambda_{t}|\lambda_{t-1})\right)\nonumber\\
	&-\sum_{t=1}^{T}(\frac{1}{\eta_{t-1}}-\frac{1}{\eta_{t-2}})\log(|\Omega_{t-1}|).\label{eq:8}
	\end{align}
	Since $\eta_t\leq\eta_{t-1}$ and $|\Omega_t|\geq1$, \eqref{eq:8} becomes
	\begin{align}
	\sum_{t=1}^T\frac{1}{\eta_{t}}\log\left(\E_{p_{t,m}}[e^{-\eta_t\phi_{m}}]\right)\geq&-\sum_{t=1}^T\phi_{t,\lambda_t(1)}-\sum_{t=1}^{T}(1-\frac{\eta_t}{\eta_{t-1}})d_t\nonumber\\
	&+\sum_{t=1}^T\frac{1}{\eta_{t-1}}\log\left(\Tau(\lambda_{t}|\lambda_{t-1})\right)\nonumber\\
	&-\frac{1}{\eta_{T-1}}\log(\max_{1\leq t\leq T}|\Omega_{t-1}|),
	\end{align}
	which concludes the proof.
	
	\section{Proof of Theorem \ref{thm:bound}}\label{pro:thm:bound}
	We combine Lemma \ref{lem:lB} and \ref{lem:uB4} to get
	\begin{align}
	\sum_{t=1}^T\left(\E_{p_{t,m}}\phi_{t,m}-\phi_{t,m_t}\right)\leq& (e-2)\sum_{t=1}^T\eta_t\E_{p_{t,m}}\phi_{t,m}^2\nonumber\\
	&+\frac{\log(\max_{1\leq t\leq T}|\Omega_{t-1}|)}{\eta_{T-1}}\nonumber\\
	&-\sum_{t=1}^T\frac{1}{\eta_{t-1}}\log(\Tau(\lambda_t|\lambda_{t-1}))\nonumber\\
	&+\sum_{t=1}^T(1-\frac{\eta_t}{\eta_{t-1}})d_t.
	\end{align}
	Since $\eta_t$ is decreasing and $\Tau(\lambda_t|\lambda_{t-1})\leq 1$, we get
	\begin{align}
	\sum_{t=1}^T\left(\E_{p_{t,m}}\phi_{t,m}-\phi_{t,m_t}\right)\leq& (e-2)\sum_{t=1}^T\eta_t\E_{p_{t,m}}\phi_{t,m}^2\nonumber\\
	&+\frac{\log(\max_{1\leq t\leq T}|\Omega_{t-1}|)}{\eta_{T-1}}\nonumber\\
	&-\frac{1}{\eta_{T-1}}\log(\Tau(\Lambda_T))\nonumber\\
	&+\sum_{t=1}^T(1-\frac{\eta_t}{\eta_{t-1}})d_t,
	\end{align}
	where $\Tau(\Lambda_T)\triangleq\prod_{t=1}^T\Tau(\lambda_t|\lambda_{t-1})$, which concludes the proof.
	
	\section{Proof of Theorem \ref{thm:bound2}}\label{pro:thm:bound2}
	From \eqref{etat} and the definition of $v_t$, we get
	\begin{align}
		\sum_{t=1}^T\eta_t\E_{p_{t,m}}\phi_{t,m}^2=&\sum_{t=1}^T\frac{\gamma}{\sqrt{V_t+\gamma^2D_t^2}}v_t\nonumber\\
		\leq&\sum_{t=1}^T\frac{\gamma}{\sqrt{V_t}}v_t\nonumber\\
		\leq&\gamma\sum_{t=1}^T\frac{V_t-V_{t-1}}{\sqrt{V_t}}\nonumber\\
		\leq&\gamma\sum_{t=1}^T(\sqrt{V_t}-\sqrt{V_{t-1}})\frac{\sqrt{V_t}+\sqrt{V_{t-1}}}{\sqrt{V_t}}\nonumber\\
		\leq&2\gamma\sum_{t=1}^T(\sqrt{V_t}-\sqrt{V_{t-1}})\nonumber\\
		\leq&2\gamma\sqrt{V_T}.\label{etaVt}
	\end{align}
	Moreover, from \eqref{etat} and the definitions of $d_t$, $D_t$, we get
	\begin{align}
		\sum_{t=1}^T(1-\frac{\eta_t}{\eta_{t-1}})d_t=&\sum_{t=1}^T\left(1-\frac{\sqrt{V_{t-1}+\gamma^2D_{t-1}^2}}{\sqrt{V_t+\gamma^2D_t^2}}\right)d_t\nonumber\\
		\leq&\sum_{t=1}^T\left(1-\frac{\sqrt{V_{t-1}+\gamma^2D_{t-1}^2}}{\sqrt{V_t+\gamma^2D_t^2}}\right)D_t\nonumber\\
		\leq&\frac{1}{\gamma}\sum_{t=1}^T\left({\sqrt{V_t+\gamma^2D_t^2}-\sqrt{V_{t-1}+\gamma^2D_{t-1}^2}}\right)\nonumber\\
		\leq&\frac{1}{\gamma}\sqrt{V_T+\gamma^2D_T^2},\label{etaDt}
	\end{align}
	Using \eqref{etaVt}, \eqref{etaDt} and the fact that $\eta_T\leq\eta_{T-1}$ in Theorem \ref{thm:bound},  we get
	\begin{align}
	\sum_{t=1}^T(\E_{p_{t,m}}\phi_{t,m}-\phi_{t,m_t})
	\leq&\frac{W(\Lambda_T)}{\gamma}\sqrt{V_T+\gamma^2D_T^2}\nonumber\\
	&+{2(e-2)\gamma\sqrt{V_T}},
	\end{align}
	where $W(\Lambda_T)\triangleq 1+\log(|\Omega_{T}|)-\log(\Tau(\Lambda_T))$ and concludes the proof.
	
	\section{Proof of Corollary \ref{cor:1}}\label{pro:cor:1}
	From Theorem \ref{thm:bound2} and concavity of the squareroot, we have
	\begin{align}
	\sum_{t=1}^T(\E_{p_{t,m}}\phi_{t,m}-\phi_{t,m_t})\leq&\frac{W(\Lambda_T)}{\gamma}\sqrt{V_T}+W(\Lambda_T)D_T\nonumber\\
	&+{2(e-2)\gamma\sqrt{V_T}},
	\end{align}
	since $W_T\geq W(\Lambda_T)$, we get
	\begin{align}
	\sum_{t=1}^T(\E_{p_{t,m}}\phi_{t,m}-\phi_{t,m_t})\leq&\frac{W_T}{\gamma}\sqrt{V_T}+W_TD_T\nonumber\\
	&+{2(e-2)\gamma\sqrt{V_T}},\label{coreq1}
	\end{align}
	We put $\gamma=\sqrt{\frac{W_T}{2(e-2)}}$ in \eqref{coreq1} and get
	\begin{align}
	\sum_{t=1}^T(\E_{p_{t,m}}\phi_{t,m}-\phi_{t,m_t})\leq&W_TD_T+{2.4\sqrt{W_TV_T}},
	\end{align}
	since $2(e-2)\leq 1.44$, which concludes the proof.
	\end{appendices}
	
\end{document}